\documentclass{article}

\usepackage{microtype}
\usepackage{graphicx}
\usepackage{subfigure}
\usepackage{booktabs} 
\usepackage{diagbox}

\usepackage{hyperref}

\usepackage[accepted]{icml2023}

\usepackage{amsmath}
\usepackage{amssymb}
\usepackage{mathtools}
\usepackage{amsthm}

\def\*#1{\mathbf{#1}}
\def\~#1{\boldsymbol{#1}}

\usepackage[capitalize,noabbrev]{cleveref}

\theoremstyle{plain}
\newtheorem{theorem}{Theorem}[section]
\newtheorem{proposition}[theorem]{Proposition}
\newtheorem{lemma}[theorem]{Lemma}

\theoremstyle{definition}
\newtheorem{definition}[theorem]{Definition}

\theoremstyle{remark}
\newtheorem{remark}[theorem]{Remark}
\newtheorem{example}[theorem]{Example}

\usepackage[textsize=tiny]{todonotes}

\icmltitlerunning{How to address  monotonicity for model risk management?}

\begin{document}

\twocolumn[
\icmltitle{How to Address Monotonicity for Model Risk Management?}



\icmlsetsymbol{equal}{*}

\begin{icmlauthorlist}
\icmlauthor{Dangxing Chen}{yyy}
\icmlauthor{Weicheng Ye}{yyy}
\end{icmlauthorlist}

\icmlaffiliation{yyy}{Zu Chongzhi Center for Mathematics and Computational Sciences, Duke Kunshan University, 
Kunshan, Jiangsu, China}
\icmlcorrespondingauthor{Dangxing Chen}{dangxing.chen@dukekunshan.edu.cn}

\icmlkeywords{Machine Learning, ICML}

\vskip 0.3in
]



\printAffiliationsAndNotice{}  

\begin{abstract}
In this paper, we study the problem of establishing the accountability and fairness of transparent machine learning models through monotonicity. Although there have been numerous studies on individual monotonicity, pairwise monotonicity is often overlooked in the existing literature. This paper studies transparent neural networks in the presence of three types of monotonicity: individual monotonicity, weak pairwise monotonicity, and strong pairwise monotonicity. As a means of achieving monotonicity while maintaining transparency, we propose the monotonic groves of neural additive models. As a result of empirical examples, we demonstrate that monotonicity is often violated in practice and that monotonic groves of neural additive models are transparent, accountable, and fair.
\end{abstract}

\section{Introduction}
There has been growing public concern over the misuse of artificial intelligence models in the absence of regulations, despite the success of artificial intelligence (AI) and machine learning (ML) in many fields \cite{radford2019language,he2016deep,chen2016xgboost}. The European Commission (EC) has proposed the Artificial Intelligence Act (AIA) \cite{EU2021act}, which represents a significant first step toward filling the regulatory void. Regulations regarding artificial intelligence should consider transparency, accountability, and fairness \cite{carlo2021AI,OCC2021model}.

Many efforts have been made to develop transparent ML models \cite{rudin2019stop,agarwal2021neural, yang2021gami, tsang2020does, hastie2017generalized,caruana2015intelligible, lou2013accurate}. A transparent model facilitates the explanation of how it makes decisions, therefore allowing us to easily verify conceptual soundness and fairness. 

Nevertheless, conceptual soundness and fairness are not necessarily guaranteed for ML models, even if they are transparent. Our focus in this paper is on monotonicity, one of the most important indicators. 
In recent years, monotonic machine learning models have received extensive research attention \cite{yanagisawa2022hierarchical,liu2020certified,milani2016fast,you2017deep,potharst2002classification,duivesteijn2008nearest}. These studies have led to a more reasonable and fair approach to ML. The majority of papers, however, focus on individual monotonicity, that is, on the fact that a model is monotonic with a particular feature. It was only recently pointed out that individual monotonicity is insufficient to summarize all relevant information \cite{chen2022monotonic,gupta2020multidimensional}. It is also important to consider pairwise monotonicity, a monotonicity that considers monotonicity between different features. Furthermore, most of these models are not necessarily transparent.



In this paper, pairwise monotonicity is explored in more detail, particularly in the context of transparent machine learning models. We divide pairwise monotonicity into two types: the pairwise monotonicity introduced in \cite{chen2022monotonic} is classified as weak pairwise monotonicity, and monotonic dominance discussed in \cite{gupta2020multidimensional} is classified as strong pairwise monotonicity. Time and severity are the two most common causes of pairwise monotonicity. In terms of time, recent information should often be considered more important than older information. For example, in credit scoring, if there is one past due, the credit score should be lower if the past due occurred recently. It is important to take into account such pairwise monotonicity in order to give people the opportunity to improve. It is important that all individuals have the opportunity to succeed without being solely based on their past behaviors. In terms of severity, some events are intrinsically more severe than others due to the nature of justice. A felony, for example, is more serious than a misdemeanor in criminal justice. It is important to maintain pairwise monotonicity as justice is an important component of fairness and a good society should have a system of reward and punishment that is fair. Furthermore, weak and strong pairwise monotonicity are distinguished based on whether two features can only be compared at the same magnitude. Strong pairwise monotonicity occurs when two features can be compared at any level. Justice usually dictates the making of such comparisons. 

Pairwise monotonicity is analyzed and its impact on statistical interactions is discussed. The traditional way to check additive separability should incorporate monotonicity constraints. Features with strong pairwise monotonicity and diminishing marginal effects should not be separated, even if data indicate otherwise. A new class of monotonic groves of neural additive models (MGNAMs) is presented to incorporate three types of monotonicity into transparent neural networks. We demonstrate empirically that pairwise monotonicities frequently occur in a wide range of fields, including finance, criminology, and healthcare. Overall, MGNAMs provide a transparent, accountable, and fair framework.

\section{Monotonicity}

For problem setup, assume we have $\mathcal{D} \times \mathcal{Y}$, where $\mathcal{D}$ is the dataset with $n$ samples and $m$ features and $\mathcal{Y}$ is the corresponding numerical values in regression and labels in classification. We assume the data-generating process 
\begin{align}
y = f(\*x) + \epsilon
\end{align}
for regression problems and 
\begin{align}
 y|\*x = \text{Bernoulli}(g^{-1}(f(\*x)))    
\end{align}
for binary classification problems, where $g^{-1}$ is the logistic function in this paper. For simplicity, we assume $\*x \in \mathbb{R}^m$. Then ML methods are applied to approximate $f$.

\subsection{Individual Monotonicity}

Throughout the paper, without loss of generality, we focus on the monotonic increasing functions. Suppose $\~{\alpha}$ is the list of all individual monotonic features and  $\neg \~{\alpha}$ its complement, then the input $\*x$ can be partitioned into $\*x = (\*x_{\~{\alpha}}, \*x_{\neg \~{\alpha}})$. 
Then we have the following definition. 
\begin{definition} \label{def:indi_mono}
We say $f$ is \textbf{individually monotonic} with respect to $\*x_{\~{\alpha}}$ if
\begin{align} \label{eq:mono_con1}
 & f(\*x_{\~{\alpha}}, \*x_{\neg \~{\alpha}}) \leq f(\*x'_{\~{\alpha}}, 
 \*x_{\neg \~{\alpha}}), \nonumber \\
 & \*x_{\~{\alpha}} \leq \*x'_{\~{\alpha}},  \forall \*x_{\~{\alpha}}, \*x_{\~{\alpha}}',\*x_{\neg \~{\alpha}} ,
\end{align}
where $\*x_{\~{\alpha}} \leq \*x_{\~{\alpha}}'$ denotes the inequality for all entries, i.e., $x_{\alpha_i} \leq x_{\alpha_i}', \forall{i}$.
\end{definition}
Here is an example of individual monotonicity. 
\begin{example} \label{eg:indi_mono}
In credit scoring, the probability of default should increase  as the number of past due increases. 
\end{example}
For a differentiable function $f$,  individual monotonicity with respect to $\*x_{\~{\alpha}}$ can be verified if 
\begin{align} \label{eq:mono_verify1}
\min_{\*x, i} \frac{\partial f (\*x)}{\partial x_{\alpha_i}} \geq 0.    
\end{align}

\subsection{Pairwise Monotonicity}

There are some features that are intrinsically more important than others in practice. Analog to \eqref{eq:mono_con1}, we partition $\*x = (x_{\beta},x_{\gamma},\*x_{\neg})$. Without loss of generality, we assume $x_{\beta}$ is more important than $x_{\gamma}$. As a result of multiple features encountering pairwise monotonicity, we record them in two lists $\*u$ and $\*v$ such that $u_i$ is more important than $v_i$. Lastly, we require all features with pairwise monotonicity also satisfy individual monotonicity.

 \subsubsection{Weak Pairwise Monotonicity}

 We classify the pairwise monotonicity introduced in \cite{chen2022monotonic} as the weak pairwise monotonicity. The definition is given as follows. 

\begin{definition}\label{def:weak_mono}
We say $f$ is \textbf{weakly monotonic} with respect to $x_{\beta}$ over $x_{\gamma}$ if 
\begin{align}\label{eq:mono_con2}
    & f(x_{\beta},x_{\gamma}+c,\*x_{\neg}) \leq f(x_{\beta}+c,x_{\gamma},\*x_{\neg}), \nonumber \\
    & \forall x_{\beta}, x_{\gamma} \text{ s.t. } x_{\beta}=x_{\gamma}, \forall \*x_{\neg}, \forall c \in \mathbb{R}^+. 
\end{align}
\end{definition}
We give an example of weak pairwise monotonicity below. 
\begin{example} \label{eg:weak_mono}
Functions should be weakly monotonic with respect to features containing current information over features containing past information. Following Example~\ref{eg:indi_mono}, let $x_{\beta}$ and $ x_{\gamma}$ count the number of past dues within two years and two years ago, then the probability of default is weakly monotonic with respect to $x_{\beta}$ over $x_{\gamma}$. 
\end{example}
Such monotonicity is considered weak due to the condition of $x_{\beta} = x_{\gamma}$. Using this condition ensures that the effects of features on the function are compared at the same magnitude, and can therefore be viewed as a more general definition. 

Suppose $f$ is differentiable and is weakly monotonic with respect to $u_i$ over $v_i$ for all i in lists $\*u$ and $\*v$, then the weak pairwise monotonicity can be verifed as
\begin{align} \label{eq:mono_verify2}
\min_{\widetilde{\*x}, i} \left( \frac{\partial f}{\partial x_{u_i}}(\widetilde{\*x}) - \frac{\partial f}{\partial x_{v_i}}(\widetilde{\*x}) \right) \geq 0.    
\end{align}
where $\widetilde{x}_{u_i} = \widetilde{x}_{v_i}$ in $\widetilde{\*x}$.

\subsubsection{Strong Pairwise Monotonicity}

In addition to the weak pairwise monotonicity, there exists a stronger condition of pairwise monotonicity. We classify the monotonic dominance introduced in \cite{gupta2020multidimensional} as the strong pairwise monotonicity. 
\begin{definition}\label{def:str_mono}
We say $f$ is \textbf{strongly monotonic} with respect to $x_{\beta}$ over $x_{\gamma}$ if 
\begin{align}\label{eq:mono_con3}
    & f(x_{\beta},x_{\gamma}+c,\*x_{\neg}) \leq f(x_{\beta}+c,x_{\gamma},\*x_{\neg}), \nonumber \\
    & \forall x_{\beta}, x_{\gamma}, \*x_{\neg}, \forall c \in \mathbb{R}^+. 
\end{align}
\end{definition}

The difference between strong/weak monotonicity is whether the condition $x_{\beta} = x_{\gamma}$ is needed. Strong monotonicity implies the impacts of increments of some features are more important than others at any point. Note that the features in Example~\ref{eg:weak_mono} are only weakly pairwise monotonic, not strongly pairwise monotonic. Adding more past dues to the credit score will have a different impact based on the number of past dues. Thus, current and past features cannot be directly compared, unless they are of equal magnitude. We provide an example of strong pairwise monotonicity below. 
\begin{example} \label{eg:stro_mono}
In criminology, an additional felony is always considered more serious than an additional misdemeanor. Therefore, the probability of recidivism should be strongly monotonic with respect to felonies over misdemeanors. 
\end{example}
Clearly, we have the following Lemma.

\begin{lemma} \label{lem:str_imply_weak}
If $f$ is strongly monotonic with respect to $x_{\beta}$ over $x_{\gamma}$, then $f$ is also weakly monotonic with respect to $x_{\beta}$ over $x_{\gamma}$. 
\end{lemma}

For a differentiable function $f$, suppose $f$ is strongly monotonic with respect to $u_i$ over $v_i$ for all i in lists $\*u$ and $\*v$, then the strong pairwise monotonicity can be verifed as
\begin{align} \label{eq:mono_verify3}
    \min_{\*x,i} \left( \frac{\partial f}{\partial x_{u_i}}(\*x) - \frac{\partial f}{\partial x_{v_i}} (\*x) \right) \geq 0.
\end{align}

Strong pairwise monotonicity is transitive and we provide the following Lemma, where proof is provided in Appendix~\ref{sec:appen_proof}. 
\begin{lemma}\label{lem:str_tran}
If $f$ is strongly monotonic with respect to $x_{\beta}$ over $x_{\gamma}$ and $x_{\gamma}$ over $x_{\delta}$, then $f$ is strongly monotonic with respect to $x_{\beta}$ over $x_{\delta}$. 
\end{lemma}

\section{Statistical Interactions}

The study of transparent machine learning models has become increasingly popular in order to improve explanation and compliance with regulatory requirements. As a general rule, we should avoid interactions between features if they do not exist in order to maintain transparency in our models. One popular class of transparent models is generalized additive models (GAMs) \cite{hastie2017generalized} of the form
\begin{align} \label{eq:GAMs}
f(\*x) = \alpha + \sum_{i=1}^m f_i(x_i). 
\end{align}
GAMs are transparent in that statistical interactions are not included. 
\cite{agarwal2021neural,caruana2015intelligible} have shown that combination of GAMs with ML models achieved high accuracy for many datasets. In this section, we discuss whether we could incorporate three types of monotonicity into GAMs. 


\subsection{Individual and Weak Pairwise Monotonicity for GAMs}

In GAMs, individual and weak pairwise monotonicity can be easily enforced. Assume that $f$ follows the GAM  \eqref{eq:GAMs} of the form and is differentiable. If $f$ is individually monotonic with respect to $x_{\alpha}$, then we need 
\begin{align}
f_{\alpha}'(x) \geq 0, \ \forall x \in \mathbb{R}.  
\end{align}
Similarly, if $f$ is weakly monotonic with respect to $x_{\beta}$ over $x_{\gamma}$, then the weak pairwise monotonicity requires that
\begin{align}
f_{\beta}'(x) \geq f_{\gamma}'(x), \forall x \in \mathbb{R}.
\end{align}
Constraints such as these can be easily implemented \cite{chen2022monotonic}. Furthermore, without statistical interactions, weak pairwise monotonicity is also transitive, as illustrated in the following Lemma with proof in Appendix~\ref{sec:appen_proof}. 
\begin{lemma}\label{lem:weak_tran}
If $f$ follows the GAM \eqref{eq:GAMs}, $f$ is weakly monotonic with respect to $x_{\beta}$ over $x_{\gamma}$ and $x_{\gamma}$ over $x_{\delta}$, then $f$ is weakly monotonic with respect to $x_{\beta}$ over $x_{\delta}$. 
\end{lemma}

\subsection{Additive Separability}

Statistical interactions can be determined by checking additive separability. 
For simplicity, suppose there are two groups: $\*x$ can be split into two components $\*x_U$ and $\*x_V$, with $U \cup V = D$ and $U \cap V = \emptyset$, where $D = \{1, \dots, m\}$. Extending it to multiple groups is straightforward. 
\begin{definition}\label{def:sep_add}
We say a function $f$ with $D$ is strictly additive separable for $U$ and $V$ if 
\begin{align}
f(\*x)  = g(\*x_U) + h(\*x_V)
\end{align}
for some functions $g$ and $h$, $U \cup V = D$, and $U \cap V = \emptyset$. 
\end{definition}

Recently, statistical interactions have been studied extensively in the existing literature \cite{sorokina2008detecting,tsang2018neural,tsang2020does}. Roughly speaking, we wish to know whether there are interactions between groups $U$ and $V$. As implied from the name, the conclusion is often drawn according to whether such interactions are statistically significant. There are many different rules to check statistical significance, as a simple example, we might consider a threshold $\epsilon$ and check whether the accuracy deteriorates if no interactions are assumed. 
\begin{align} \label{eq:stat_inter}
  &\textbf{Verify additive separability:} \nonumber \\
  &|\text{Acc}(f(\*x))-\text{Acc}(g(\*x_U)+h(\*x_V))| < \epsilon.
\end{align}
If the criteria are satisfied, it seems reasonable to conclude that there are no interactions between $U$ and $V$. When it comes to GAMs, if a GAM achieves similar accuracy as the black-box ML model, we may conclude that no interaction is necessary.

\subsection{Additive Separability in the Presence of Monotonicity}

In the context that monotonicity is required, we should add the monotonicity into the requirement of additive separability. That motivates us to modify the rule of Equation~\ref{eq:stat_inter}.
\begin{align} \label{eq:stat_inter_mono}
  &\textbf{Verify additive separability with monotonicity:} \nonumber \\
  &|\text{Acc}(f(\*x))-\text{Acc}(g(\*x_U)+h(\*x_V))| < \epsilon, \\
  & \text{$f$ and $g+h$ have required monotonicity.} \nonumber
\end{align}

For statistical interactions with monotonicity, monotonicity constraints for $g+h$ are essential, since we may not have sufficient data for statistically significant results. In spite of this, neglecting such a statistical interaction may have catastrophic consequences. 
To illustrate our idea, consider the following example of credit scoring. 
\begin{example}
Suppose $\*x = (x_{\beta},x_{\gamma})$ where $x_{\beta}$ counts the number of past dues of more than 60 days and $x_{\gamma}$ counts the number of past dues between 30 and 59 days. Assume $f$ calculates the probability of defaults. Clearly, $f$ should be strongly monotonic with respect to $x_{\beta}$ over $x_{\gamma}$. For simplicity, consider the values of $f$ in the region where $x_{\beta}+x_{\gamma} \leq 2$. Suppose the true function $f$ and an additive approximation $\widetilde{f} = f_1 + f_2$ are given in Table~\ref{tab:stat_inter_fail}. If there are no data for $\*x = (1,1)$, then $\widetilde{f}$ exactly fits $f$ in all training data. According to the criteria \eqref{eq:stat_inter}, $x_{\beta}$ and $x_{\gamma}$ can be well separated. However, $\widetilde{f}$ doesn't have strong pairwise monotonicity and $\widetilde{f}(1,1) > \widetilde{f}(2,0)$ causes algorithmic unfairness. Furthermore, such rules could encourage people with $\*x = (1,1)$ to wait for an additional month to pay back to change to $\*x = (2,0)$ in order to obtain a lower probability of default, and therefore higher credit score. Even worse, ML models might not recognize this from data in the long run, as people would intentionally avoid the state $\*x = (1,1)$. Data does not reveal such a problem, thus it must be considered in advance. 


\begin{table}[t]
\caption{Comparison between $f$ with strong pairwise monotonic features and an additive approximation $\widetilde{f} = f_1 + f_2$. $f$ is strongly monotonic with respect to $x_{\beta}$ over $x_{\gamma}$. $\widetilde{f}$ violates strong pairwise monotonicity at $\*x=(1,1)$.}
\label{tab:stat_inter_fail}
\vskip 0.15in
\begin{center}
\begin{small}
\begin{sc}
\begin{tabular}{l|cccr}
\toprule
 True $f$ & & & \\ 
 \midrule
 2 & 0.4 &   &  \\ 
 1 & 0.3 & \textbf{0.35} &  \\ 
 0 & 0 & 0.2 & 0.3 \\ \hline
 $x_{\beta} \backslash x_{\gamma}$ & 0 & 1 & 2 \\
 \midrule
 $\widetilde{f}=f_1+f_2$ & & & \\
 \midrule
 2 & 0.4 &   &  \\ 
 1 & 0.3 &  {\color{red} \textbf{0.5}} &  \\ 
 0 & 0 & 0.2 & 0.3 \\ \hline
 $x_{\beta} \backslash x_{\gamma}$ & 0 & 1 & 2 \\
\bottomrule
\end{tabular}
\end{sc}
\end{small}
\end{center}
\vskip -0.1in
\end{table}

\end{example}

\subsection{In the Presence of Strong Pairwise Monotonicity}

We argue that there exists a common situation in which features with strong pairwise monotonicity cannot be separated, except in the trivial case. Let us consider the following proposition, whereas the proof is in Appendix~\ref{sec:appen_proof}. 
\begin{proposition}\label{prop:nonseparable}
Suppose $f$ takes the GAM form \eqref{eq:GAMs}, $f$ is differentiable, individually monotonic with respect to $x_{\beta}$ and $x_{\gamma}$, and strongly monotonic with respect to $x_{\beta}$ over $x_{\gamma}$. If there exists $x^*$ such that $f_{\beta}'(x^*) = 0$, then $f_{\gamma}(x)$ is a constant function.
\end{proposition}
According to the Proposition, under such additive forms, $f_{\gamma}$ is a constant function, which can be inconsistent with reality. Sadly, such phenomena are common in practice, and one of the most common causes is diminishing marginal effects. We provide the definition below. 

\begin{definition}
Suppose $\*x = (x_{\alpha},\*x_{\neg \alpha})$. We say a differentiable function $f$ has the \textbf{diminishing marginal effect (DME)} with respect to $x_{\alpha}$ if followings hold
\begin{enumerate}
    \item $\frac{\partial}{\partial x_{\alpha}} f(x_{\alpha},\*x_{\neg \alpha}) > 0$
    \item $\frac{\partial^2}{\partial x_{\alpha}^2} f(x_{\alpha},\*x_{\neg \alpha}) < 0$
    \item $\lim_{x_{\alpha} \rightarrow \infty}$ $\frac{\partial}{\partial x_{\alpha}} f(x_{\alpha},\*x_{\neg \alpha}) = 0$.
\end{enumerate}
\end{definition}
As a matter of fact, DMEs are quite common in practice. For example, the Cobb-Douglas utility function, $u(x,y) = x^{a} y^{1-a}$ with $0<a<1$, is commonly used to illustrate diminishing marginal utility in economics. 




Proposition~\ref{prop:nonseparable} suggests that DMEs may prevent us from separating features with strong pairwise monotonicity. Features with strong pairwise monotonicity that exhibits DME patterns must be assumed to be non-separable at the time of its emergence. Therefore, GAMs are insufficient to incorporate strong pairwise monotonicity in this case. 


\subsection{Implications on Binary Features}

There is an exception to the previous analysis, which is when features are binary  since DMEs do not apply. In this case, we have the following Lemma, whereas the proof is left in Appendix~\ref{sec:appen_proof}. 
\begin{lemma} \label{lem:binary_mono}
For binary features, weak pairwise monotonicity coincides with strong pairwise monotonicity. 
\end{lemma}
In this case, features can still be additive separable in the linear form. Consider the linear regression of the following form for simplicity
\begin{align*}
f(\*x) = \alpha + \sum_{i=1}^m \beta_i x_i.
\end{align*}
Suppose $f$ is monotonic with respect to $x_{\gamma}$ over $x_{\delta}$, then we require $\beta_{\gamma}>\beta_{\delta}$, and the additive separability can be achieved.

\section{Monotonic Groves of Neural Additive Models}

There has been an increasing demand for transparent models recently. In this direction, 
Neural additive models (NAMs) \cite{agarwal2021neural} and its monotonic version \cite{chen2022monotonic} provide the most transparent neural networks by avoiding statistical interactions, and have been very successful. NAMs have assumed that each $f_i$ in Equation \eqref{eq:GAMs} is parametrized by neural networks (NNs). Despite their success, they cannot handle strong pairwise monotonicity, as discussed above. We aim to develop a new model that will maintain transparency to the greatest extent possible, in the manner of NAMs, as well as incorporate strong pairwise monotonicity. 
Thus, we consider a more general form, namely the groves of neural additive models (GNAMs), similar to \cite{sorokina2008detecting},
\begin{align} \label{eq:GNAM}
    f(\*x) = \alpha  + \sum_{p: p \in P} f_p(x_p) + \sum_{q: q \in Q} f_q(\*x_q),
\end{align}
$f_p$ and $f_q$ are parametrized by NNs. There exists five types of features:
\begin{itemize}
\item Nonmonotonic features
\item Features with only individual monotonicity
\item Features with only weak pairwise monotonicity
\item Features with only strong pairwise monotonicity
\item Features with both strong and weak pairwise monotonicity
\end{itemize}
The first three types of features are trained by 1-dimensional functions $f_p$, just like monotonic NAMs (MNAMs) \cite{chen2022monotonic}. Different from MNAMs, $\*x_q$ can be higher-dimensional. For the last two types, when there is strong pairwise monotonicity involved, features with pairwise monotonicity should be grouped together in $q$. Note we group features with both strong and weak pairwise monotonicity to avoid unfair comparisons. Detailed explanations can be found in Appendix~\ref{sec:mono_remark}.

Regularized algorithms are used to enforce monotonicity. In GNAMs' architecture, motivated by conditions \eqref{eq:mono_verify1}, \eqref{eq:mono_verify2}, and \eqref{eq:mono_verify3}, we consider the optimization problem:
\begin{align} \label{eq:MNAM_loss}
\min_{\~{\Theta}} \ell(\~{\Theta}) + \lambda_1 h_1(\~{\Theta}) + \lambda_2 h_2(\~{\Theta}) + \lambda_3 h_3(\~{\Theta}),
\end{align}
where $\ell(\~{\Theta})$ is the mean-squared error for regressions and log-likelihood function for classifications, and
\begin{itemize}
\item Individual monotonicity: suppose $\~{\alpha}$ is the list of individual monotonic features, then
\begin{align*}
h_1(\~{\Theta}) &= \sum_{\alpha \in \~{\alpha}} \int_{\mathbb{R}^m} \max\left(0,-\frac{\partial f(\*x; \~{\Theta})}{\partial x_{\alpha}}  \right)^2  \ d \*x.
\end{align*}
\item Weak pairwise monotonicity: suppose $\*u$ and $\*v$ are weak pairwise monotonic lists such that $f$ is weakly monotonic with respect to $u_i$ over $v_i$, then
\begin{align*}
h_2(\~{\Theta}) &= \sum_{i=1}^{|\*u|} \int_{\mathbb{R}^{m-1}}  \max\left(0, \Delta f(\widetilde{\*x}_i,\~{\Theta})  \right)^2  \ d\widetilde{\*x}_i
\end{align*}
where 
\begin{align*}
\Delta f(\widetilde{\*x}_i,\~{\Theta}) = -\frac{\partial f (\widetilde{\*x}_i;\~{\Theta})}{\partial x_{u_i}} + \frac{\partial f (\widetilde{\*x}_i;\~{\Theta})}{\partial x_{v_i}}
\end{align*}
and $x_{u_i} = x_{v_i}$ in $\widetilde{\*x}_i$.
\item Strong pairwise monotonicity: suppose $\*y$ and $\*z$ are strong pairwise monotonic lists such that $f$ is strongly monotonic with respect to $y_i$ over $z_i$, then
\begin{align*}
h_3(\~{\Theta}) &= \sum_{i=1}^{|\*y|} \int_{\mathbb{R}^{m}}  \max\left(0, \Delta f_i(\*x,\~{\Theta})  \right)^2  \ d\*x 
\end{align*}
where 
\begin{align*}
\Delta f_i(\*x,\~{\Theta}) = -\frac{\partial f (\*x;\~{\Theta})}{\partial x_{y_i}} + \frac{\partial f (\*x;\~{\Theta})}{\partial x_{z_i}}.
\end{align*}
\end{itemize}
In the GNAM's architecture \eqref{eq:GNAM}, computational dimensions can be reduced. For example, when calculating partial derivatives for features in the group $q$, it is sufficient to evaluate $\partial f_q$  instead of $\partial f$. In practice, we replace the integral with the equispaced discrete approximations. In the optimization procedure, we also replace all $\max(0,\cdot)$ with $\max(\epsilon,\cdot)$.

We gradually increase $\lambda_1$, $\lambda_2$, and $\lambda_3$ until penalty terms vanish. The two-step procedure is summarized in Algorithm~\ref{alg:MGNAM}. We refer to the GNAM that satisfies all required monotonic constraints \eqref{eq:mono_con1}, \eqref{eq:mono_con2}, and \eqref{eq:mono_con3} as the monotonic groves of neural additive model (MGNAM). 


\begin{algorithm}[tb]
  \caption{Monotonic Groves of Neural Additive Model}
  \label{alg:MGNAM}
\begin{algorithmic}
  \STATE {\bfseries Initialization:} $\lambda_1=\lambda_2=\lambda_3=0$, the architecture of the GNAM ($P$ and $Q$)
  \STATE Train a GNAM by \eqref{eq:MNAM_loss}
  \WHILE{$\min(h_1,h_2,h_3)>0$} 
    \STATE Increase $\lambda_i$ if $h_i>0$
    \STATE Retrain the GNAM by \eqref{eq:MNAM_loss}.
  \ENDWHILE
\end{algorithmic}
\end{algorithm}

\section{Empirical Examples}

This section evaluates the performance of models for a variety of datasets in different fields, including finance, criminology, and health care. We compare fully-connected neural networks (FCNNs), neural additive models (NAMs), monotonic neural additive models (MNAMs), and monotonic groves of neural additive models (MGNAMs). For MNAMs, strong pairwise monotonicity is replaced by weak pairwise monotonicity. We use FCNNs to check the accuracy of black-box ML models and NAMs/MNAMs for visualizations. We do not consider other models here as the general comparison of accuracy is not our focus, but the conceptual soundness and fairness. More details of the dataset, models, and experiments setup are provided in Appendix~\ref{sec:empirical_example}.

\subsection{Finance - Credit Scoring}

In credit scoring, statistical models are used to assess an individual's creditworthiness. A popularly used dataset is the Kaggle credit score dataset \footnote{https://www.kaggle.com/c/GiveMeSomeCredit/overview}. In this dataset, we have included three delinquency features that quantify the number of past dues and their duration: 30-59 days, 60-89 days, and 90+ days. To demonstrate the strong pairwise monotonicity of this dataset, we focus on these three features. Without loss of generality, we denote them as $x_1$, $x_2$, and $x_3$. When an additional past due exceeds 90 days, the system should take it much more seriously than when it exceeds 60-89 days, which should take it much more seriously than when it exceeds 30-59 days. We, therefore, impose strong pairwise  monotonicity on this order. In the event that such strong pairwise monotonicity is violated, customers with longer past dues could have a higher credit score, thereby causing algorithmic unfairness. In addition, customers with shorter past dues may wish to delay their payments in order to increase their credit score. 

A summary of the model performance is provided in Table~\ref{tab:GMSC_result}. There is no significant difference in accuracy between the different methods, indicating that transparent neural networks are sufficient for this dataset. 


\begin{table}[t]
\caption{Model performance of the GMSC dataset. All ML models perform similarly.}
\label{tab:GMSC_result}
\vskip 0.15in
\begin{center}
\begin{small}
\begin{sc}
\begin{tabular}{l|cccr}
\toprule
    Model/Metrics  & Classification error & AUC  \\ 
    \midrule
    FCNN & $6.6\%$ & $79.5\%$ \\ 
    NAM & $6.6\%$ & $79.8\%$ \\ 
    MNAM & $6.6\%$ & $80.0\%$ \\ 
    MGNAM & $6.5\%$ & $80.2\%$ \\ 
\bottomrule
\end{tabular}
\end{sc}
\end{small}
\end{center}
\vskip -0.1in
\end{table}

Next, we evaluate conceptual soundness and fairness. For simplicity, we focus on the number of past dues in each period that are less than or equal to two, that is, $0 \leq x_1, x_2, x_3 \leq 2$. We will begin by examining the result of the NAM since it is straightforward to visualize. A comparison of the associated functions is provided in Figure~\ref{fig:GMSC_vio}. The pairwise monotonicity is clearly violated when there is more than one past due. For example, the feature with 30-59 days past due becomes more important than the feature with 60-89 days past due. Then, we evaluate the MNAM with function values in Table~\ref{tab:GMSC_MNAM}. Both individual monotonicity and weak pairwise monotonicity are satisfied. But when statistical interactions are involved for large x, the strong pairwise monotonicity is violated. As an example, consider an applicant who has three past dues, with $x_3=1$. If $(x_1,x_2,x_3)=(0,2,1)$, then it should be punished more severely than $(x_1,x_2,x_3)=(1,1,1)$; however, according to the MNAM, $f_1(0)+f_2(2)+f_3(1)=3.4$, which is less than $f_1(1)+f_2(1)+f_3(1)=3.9$. Therefore, based on the MNAM, for the person with (0,1,1), if the applicant did not pay for one month and received (1,1,1), then he or she should wait and pay one payment in an additional month to achieve (0,2,1) for a higher credit score (lower probability of default). Clearly, the fairness of this situation has been violated.

We then examine the result of the MGNAM. We are interested in knowing if delinquency features can be separated additively. We plot the marginal probability of default as a function of $x_1-x_3$ in Figure~\ref{fig:GMSC_DME}. The presence of DMEs is evident. By Proposition~\ref{prop:nonseparable}, we cannot separate these three features additively and therefore group them together. The values of $f_q(x_1,x_2,x_3)$ calculated by MGNAM are shown in Table~\ref{tab:GMSC_GNAM}. The table provides confidence to model users by verifying all monotonicity is achieved. It should be emphasized that without satisfying monotonicity, even the most accurate ML model will not be accepted. Furthermore, the transparent nature of the MGNAM makes it easier to verify conceptual soundness and fairness, which are difficult to achieve with black-box machine learning models.  


\begin{figure}[ht]
\vskip 0.2in
\begin{center}
\centerline{\includegraphics[width=\columnwidth]{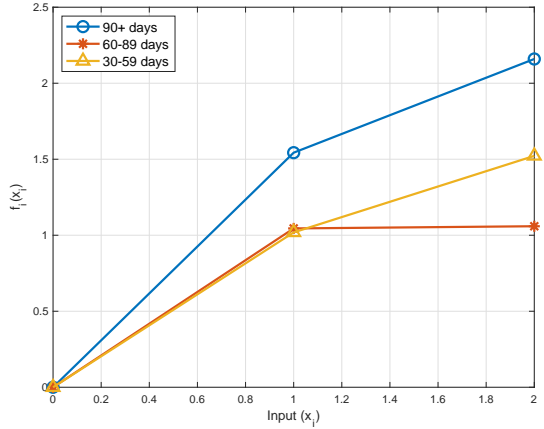}}
\caption{Comparison of functions associated with the number of past dues in different periods by the NAM for the GMSC dataset. Monotonicity is violated between 30-59 days and 60-89 days.}
\label{fig:GMSC_vio}
\end{center}
\vskip -0.2in
\end{figure}


\begin{figure}[ht]
\vskip 0.2in
\begin{center}
\centerline{\includegraphics[width=\columnwidth]{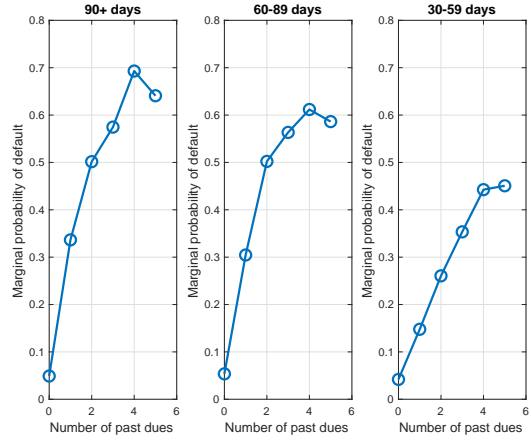}}
\caption{Marginal probability of defaults with respect to $x_1-x_3$ in the GMSC dataset. Diminishing marginal effects are observed.}
\label{fig:GMSC_DME}
\end{center}
\vskip -0.2in
\end{figure}


\begin{table}[t]
\caption{Function values for $x_1,x_2,x_3$ by the MNAM in the GMSC dataset. $f$ is weakly monotonic with respect to $x_3$ over $x_2$ and $x_2$ over $x_1$. Individual and weak pairwise monotonicity are preserved.}
\label{tab:GMSC_MNAM}
\vskip 0.15in
\begin{center}
\begin{small}
\begin{sc}
\begin{tabular}{l|cccr}
\toprule
    $f \backslash x$ & 0 & 1 & 2 \\ \hline
    $f_1$  & $0$ & $0.8$ & $1.0$  \\ 
    $f_2$ & 0 & 1.4 & 1.7 \\ 
    $f_3$ & 0 & 1.7 & 2.2 \\ 
\bottomrule
\end{tabular}
\end{sc}
\end{small}
\end{center}
\vskip -0.1in
\end{table}


\begin{table}[t]
\caption{Function values for $x_1,x_2,x_3$ by the MGNAM in the GMSC dataset. $f$ is strongly monotonic with respect to $x_3$ over $x_2$ and $x_2$ over $x_1$. Monotonicity is preserved.}
\label{tab:GMSC_GNAM}
\vskip 0.15in
\begin{center}
\begin{small}
\begin{sc}
\begin{tabular}{l|cccr}
\toprule
    $x_3=0$ & & & \\ \hline
    $x_1 \backslash x_2$  & $0$ & $1$ & $2$  \\ \hline
    $0$ & 0 & 1.7 & 2.3 \\ 
    $1$ & 1.7 & 2.3 & 2.8 \\ 
    $2$ & 2.3 & 2.8 & 3.2 \\ \hline
    $x_3=1$ & & & \\ \hline
    $x_1 \backslash x_2$  & $0$ & $1$ & $2$  \\ \hline
    $0$ & 2.2 & 2.7 & 3.2 \\ 
    $1$ & 2.7 & 3.1 & 3.5 \\ 
    $2$ & 3.1 & 3.5 & 3.7 \\ \hline
    $x_3=2$ & & & \\ \hline
    $x_1 \backslash x_2$  & $0$ & $1$ & $2$  \\ \hline
    $0$ & 3.1 & 3.4 & 3.7 \\ 
    $1$ & 3.4 & 3.6 & 3.8 \\ 
    $2$ & 3.6 & 3.8 & 3.9 \\ 
\bottomrule
\end{tabular}
\end{sc}
\end{small}
\end{center}
\vskip -0.1in
\end{table}

\subsection{Criminal Justice - COMPAS}

The COMPAS scoring system was developed to predict recidivism risk and has been scrutinized for its racial bias \cite{angwin2016machine,dressel2018accuracy,tan2018distill}. In 2016, ProPublica published recidivism data for defendants in Broward County, Florida \cite{Propublica2016COMPAS}. We focus on the simplified cleaned dataset provided in \cite{dressel2018accuracy}. Race and gender unfairness have been extensively studied in the past \cite{foulds2020intersectional,kearns2019empirical,kearns2018preventing,hardt2016equality}. Our focus is on the potential unfairness associated with types of offenses. Specifically, a felony is considered more serious than a misdemeanor. Without loss of generality, assume $x_1$ counts the number of past misdemeanors and $x_2$ counts the number of past felonies. Due to this, we ask that the probability of recidivism be strongly monotonic with respect to $x_2$ over $x_1$. Criminals may consider turning a misdemeanor into a felony in the future if this strong pairwise monotonicity is violated. 

Model performance is summarized in Table~\ref{tab:COMPAS_result}. The performance of all methods is similar. In this regard, algorithmic fairness is more important than accuracy when it comes to the dataset. 


\begin{table}[t]
\caption{Model performance of the COMPAS dataset. All ML models perform similarly.}
\label{tab:COMPAS_result}
\vskip 0.15in
\begin{center}
\begin{small}
\begin{sc}
\begin{tabular}{l|cccr}
\toprule
    Model/Metrics  & Classification error & AUC  \\ 
    \midrule
    FCNN & $33.8\%$ & $71.9\%$ \\
    NAM & $34.1\%$ & $71.8\%$ \\ 
    MNAM & $33.5\%$ & $71.7\%$ \\ 
    MGNAM & $34.3\%$ & $71.9\%$ \\ 
\bottomrule
\end{tabular}
\end{sc}
\end{small}
\end{center}
\vskip -0.1in
\end{table}

Next, we evaluate conceptual soundness and fairness. For simplicity's sake, we restrict ourselves to a maximum of three charges per type. Regarding the architecture of the MGNAM, the diminishing marginal effect is clearly observed for the felony in Figure~\ref{fig:COMPAS_DME}, therefore we should group the felony and misdemeanor together, based on Proposition~\ref{prop:nonseparable}. Due to the fact that there are only two features in the group, function values are calculated and compared in tables~\ref{tab:COMPAS_GNAM}. For small values of $x_1$ and $x_2$, functions behave reasonably in the NAM. For larger values, it immediately violates pairwise monotonicity. The individual monotonicity of $x_2$ is violated when the value of $x_1$ is fixed. Furthermore, the function contribution is only 0.37 when there are three past felonies ($x_2=3$), whereas the function value is 0.65 when there is one felony and one misdemeanor ($x_1=x_2=1$). Compared to the first case, the value is almost doubled, which is a serious violation. Then, we evaluate the MNAM. Both individual monotonicity and weak pairwise monotonicity are satisfied. But when statistical interactions are involved for large x, the strong pairwise monotonicity is violated. Consider the example of $(x_1,x_2) = (0,2)$ which should be punished more severely than $(1,1)$. However, according to the MNAM, the value of the function at $(0,2)$ is $0.37$, which is less than the value at $(1,1)$ as $0.50$. Consequently, a person who commits one felony and one misdemeanor will be punished more severely than a person who commits two felonies. There is a serious violation of the principle of fairness in this situation. Additionally, if someone with one felony commits another crime, he or she may consider it to be a felony rather than a misdemeanor, leading to difficulties in society. In our model, this issue has been avoided, since the value of the function at $(0,2)$ is $0.54$, which is larger than $0.53$ at $(1,1)$. There are many other similar examples of violations. In the absence of such strong pairwise monotonicity, the algorithm should not be used. 


\begin{figure}[ht]
\vskip 0.2in
\begin{center}
\centerline{\includegraphics[width=\columnwidth]{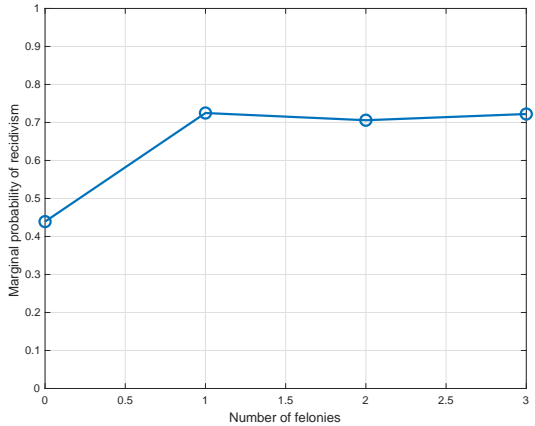}}
\caption{Marginal probability of recidivism with respect to the number of felonies in the COMPAS dataset. The diminishing marginal effect is observed.}
\label{fig:COMPAS_DME}
\end{center}
\vskip -0.2in
\end{figure}


\begin{table}[t]
\caption{Function values for $x_1,x_2$ by the MNAM and MGNAM of the COMPAS dataset. There are multiple violations of monotonicity for the NAM, for example, between $(2,2)$ and $(2,3)$, and between $(0,3)$ and $(1,1)$. Violations are also observed for the MNAM, for example, between $(0,2)$ and $(1,1)$. The MGNAM preserves monotonicity.}
\label{tab:COMPAS_GNAM}
\vskip 0.15in
\begin{center}
\begin{small}
\begin{sc}
\begin{tabular}{l|cccr}
\toprule
    MGNAM & & & & \\ \hline
    $x_1 \backslash x_2$  & $0$ & $1$ & $2$ & 3  \\ \hline
    $0$ & 0 & 0.35 & 0.54 & 0.56  \\ 
    $1$ & 0.21 & 0.53 & 0.56 & 0.56 \\ 
    $2$ & 0.49 & 0.55 & 0.56 & 0.56 \\ 
    3 & 0.55 & 0.56 & 0.56 & 0.56 \\ \hline
    NAM & & & & \\ \hline
    $x_1 \backslash x_2$  & $0$ & $1$ & $2$ & 3  \\ \hline
    $0$ & 0 & 0.41 & 0.40 & {\color{red} \textbf{0.37}}  \\ 
    $1$ & 0.24 & {\color{red} \textbf{0.65}} & 0.65 & 0.62 \\ 
    $2$ & 0.32 & 0.72 & {\color{blue} \textbf{0.72}} & {\color{blue} \textbf{0.69}} \\
    3 & 0.33 & 0.74 & 0.73 & 0.70 \\ \hline
    MNAM & & & & \\ \hline
    $x_1 \backslash x_2$  & $0$ & $1$ & $2$ & 3  \\ \hline
    $0$ & 0 & 0.33 & {\color{purple} \textbf{0.37}} & 0.37  \\ 
    $1$ & 0.17 & {\color{purple} \textbf{0.50}} & 0.54 & 0.54 \\ 
    $2$ & 0.19 & 0.53 & 0.57 & 0.57 \\ 
    3 & 0.20 & 0.53 & 0.57 & 0.57 \\ 
\bottomrule
\end{tabular}
\end{sc}
\end{small}
\end{center}
\vskip -0.1in
\end{table}

\subsection{Healthcare - Heart Failure Clinical Records}

This dataset \cite{ahmad2017survival,chicco2020machine} contains the medical records of 299 patients who had heart failure, collected during their follow-up period, where each patient profile has 13 clinical features. This study aims to predict the survival of patients suffering from heart failure. Conceptual soundness is a very important aspect of health datasets. With limited dataset, machine learning models are very easy to overfit, which can be mitigated by imposing constraints. In the case that one needs to determine the priority of patients, then fairness is also a very important factor. For this dataset, we focus on four features: smoking, anemia, high blood pressure, and diabetes. Without loss of generality, we denote them as $x_1$, $x_2$, $x_3$, and $x_4$. Anemia, high blood pressure, and diabetes are considered to be more serious health risks than smoking. Thus, $f$ should be monotonic with respect to $x_2-x_4$ over $x_1$. Due to the fact that they are all binary features, strong monotonicity is the same as weak monotonicity, by Lemma~\ref{lem:binary_mono}. 

A summary of the results is provided in Table~\ref{tab:heart_result}. Since the NAM performs similarly to the FCNN and features associated with pairwise monotonicity are only binary, we do not consider interactions, and the MGNAM coincides with the MNAM. The MGNAM also has a similar level of accuracy. 


\begin{table}[t]
\caption{Model performance of the heart dataset. All ML models perform similarly. }
\label{tab:heart_result}
\vskip 0.15in
\begin{center}
\begin{small}
\begin{sc}
\begin{tabular}{l|cccr}
\toprule
    Model/Metrics  & Classification error & AUC  \\ 
    \midrule
    FCNN & $20.3\%$ & $87.0\%$ \\ 
    NAM & $18.9\%$ & $89.8\%$ \\ 
    MGNAM & $17.6\%$ & $90.6\%$ \\ 
\bottomrule
\end{tabular}
\end{sc}
\end{small}
\end{center}
\vskip -0.1in
\end{table}

Next, we evaluate conceptual soundness and fairness. For blood and diabetes in the NAM, both individual and pairwise monotonicity are violated, as shown in Figure~\ref{fig:heart_blood} and Figure~\ref{fig:heart_diabetes}. This problem has been avoided by MGNAM. According to the NAM, high blood pressure and diabetes are actually beneficial for survival. Furthermore, smoking is more dangerous than both of them. 

\begin{figure}[h]
\begin{subfigure}
    \centering
    \includegraphics[scale=0.4]{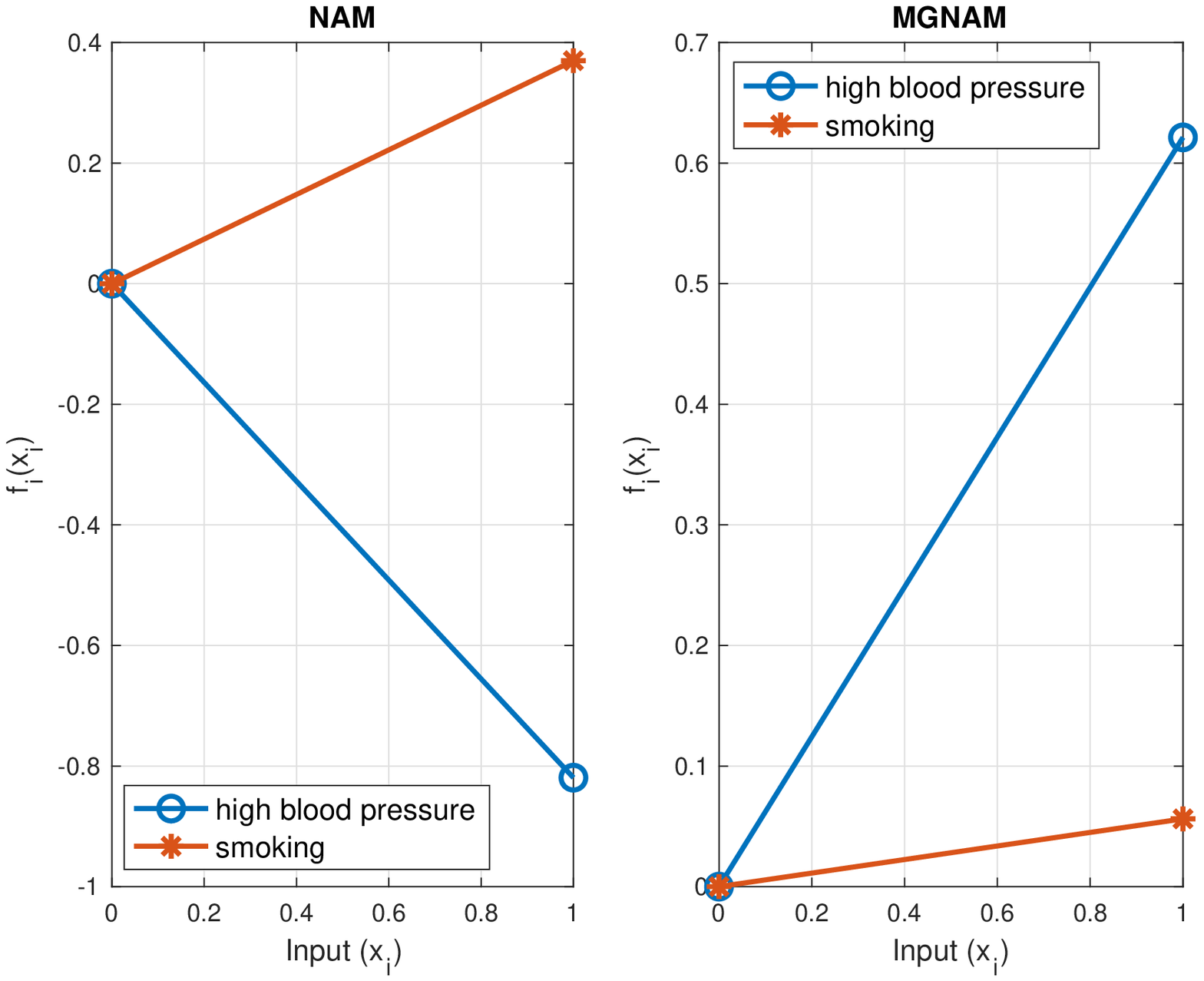}
    \caption{Comparison of blood and smoking functions by the NAM and the MGNAM for the heart dataset. The individual and pairwise monotonicity are both violated by the NAM. }
    \label{fig:heart_blood}
\end{subfigure}%
\vskip 0.2in
\begin{subfigure}
    \centering
    \includegraphics[scale=0.4]{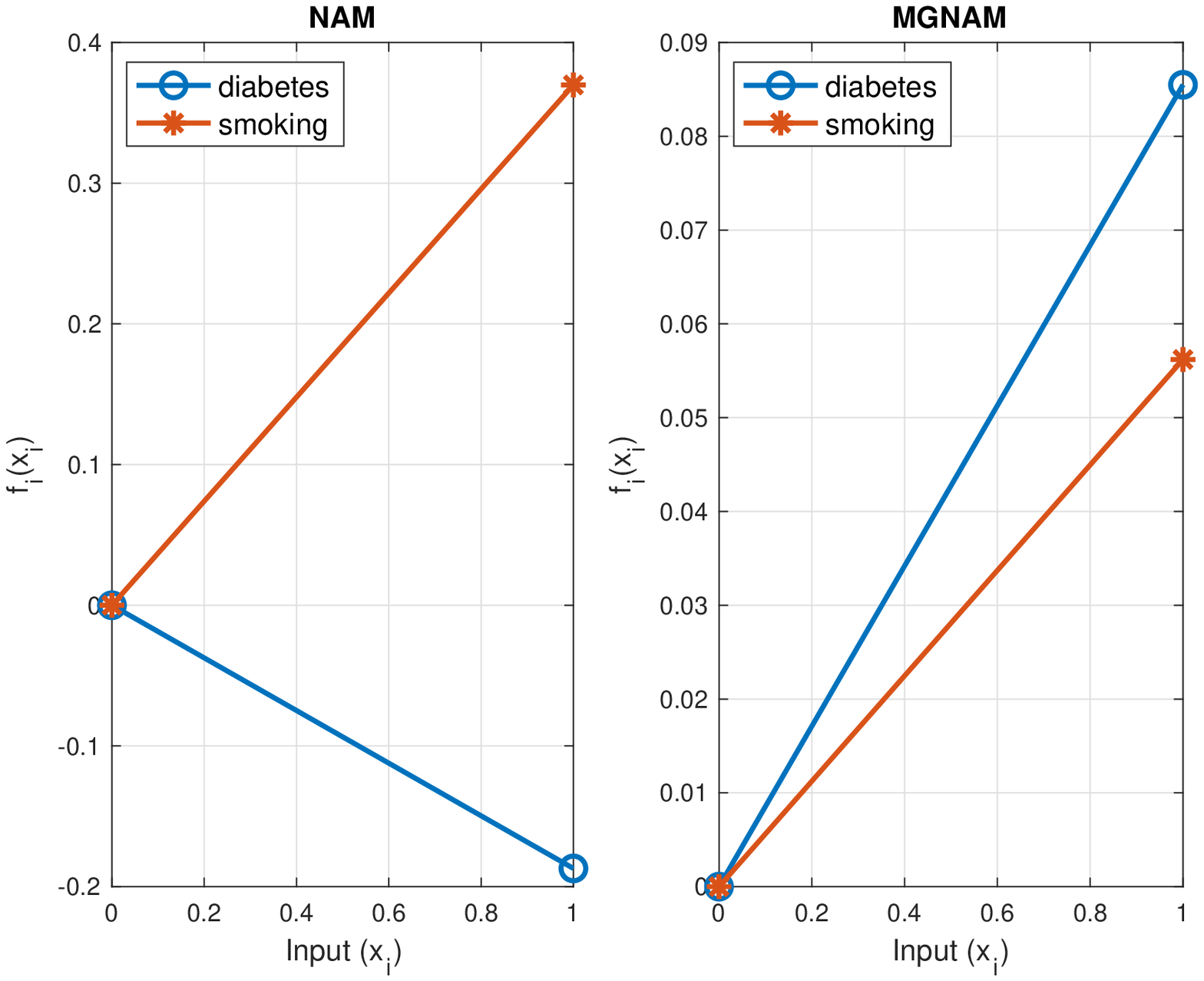}
    \caption{Comparison of diabetes and smoking functions by the NAM and the MGNAM for the diabetes dataset. The individual and pairwise monotonicity are both violated by the NAM. }
    \label{fig:heart_diabetes}
\end{subfigure}
\vskip -0.2in
\end{figure}

\section{Related work}

\textbf{Monotonic Models}: 
 Most of previous work \cite{yanagisawa2022hierarchical,liu2020certified,milani2016fast,you2017deep} focus on individual monotonicity. Weak pairwise monotonicity is considered in \cite{chen2022monotonic} and strong pairwise monotonicity is considered in \cite{gupta2020multidimensional}. Our paper has considered three types of monotonicity. 
 
 \textbf{Transparent Models}: 
 There has been enormous literature on designing transparent machine learning models. \cite{agarwal2021neural,chen2022monotonic,yang2021gami,lou2012intelligible} starts with transparent generalized additive models. Another direction specifies neural network models based on statistical interactions \cite{janizek2021explaining,tsang2018neural,tsang2020does,tsang2018detecting}. However, these approaches haven't yet included the discussion between monotonicity and transparency.

\section{Conclusion}

In this paper, we analyze three types of monotonicity and propose monotonic groves of neural additive models (MGNAMs) for transparency and monotonicity.

There are many avenues for future directions. First, the regularized algorithm with discretized integrals in the penalty functions in order to enforce monotonicity. It is possible to achieve high accuracy with continuous features by using a large number of points, however, certification is not yet available for three types of monotonicity. Second, there are many applications in which these integrals are appropriate when the dimensions of pairwise monotonic features are small. Nevertheless, there is the possibility of having a large collection of pairwise monotonic features in some contexts. In the future, we plan to investigate possible fast algorithms for implementing pairwise monotonicity. Third, in the spirit of neural additive models, we keep MGNAM architectures as simple as possible to preserve the transparency of models. There is, however, a possibility that some datasets will exhibit other interactions. The detection of statistical interactions will be studied in the future in the presence of three types of monotonicity. Further, conjoint measurement and multiple criteria decision analysis \cite{bouyssou2016conjoint,grabisch2018monotone} are also concerned with transparency when dealing with complex statistical interactions. The inclusion of some analysis will be of interest.

\section*{Acknowledgements}

We are grateful to the reviewers for their valuable and constructive comments.


\bibliography{example_paper}
\bibliographystyle{icml2023}

\newpage
\appendix
\onecolumn
\section{Appendix}

\subsection{Proof} \label{sec:appen_proof}

\begin{proof}[Proof of Lemma~\ref{lem:str_tran}]

We partition $\*x = (x_{\beta},x_{\gamma},x_{\delta},\*x_{\neg})$, based on Definition~\ref{def:str_mono}, we have
\begin{align}
& f(x_{\beta},x_{\gamma}+c,x_{\delta},\*x_{\neg}) \leq f(x_{\beta}+c,x_{\gamma},x_{\delta},\*x_{\neg}), \\
& f(x_{\beta},x_{\gamma},x_{\delta}+c,\*x_{\neg}) \leq 
f(x_{\beta},x_{\gamma}+c,x_{\delta},\*x_{\neg}), \\
& \forall x_{\beta},x_{\gamma},x_{\delta},\*x_{\neg}, \forall c \in \mathbb{R}^+.
\end{align}
These imply that
\begin{align}
f(x_{\beta},x_{\gamma},x_{\delta}+c,\*x_{\neg}) \leq f(x_{\beta}+c,x_{\gamma},x_{\delta},\*x_{\neg}).
\end{align}
Thus, we conclude. 
\end{proof}

\begin{proof}[Proof of Lemma~\ref{lem:weak_tran}]

If $f$ is weakly monotonic with respect to $x_{\beta}$ over $x_{\gamma}$ and $x_{\gamma}$ over $x_{\delta}$, we have
\begin{align}
f_{\beta}'(x) \geq f_{\gamma}'(x), f_{\gamma}'(x) \geq f_{\delta}'(x), \ \forall x \in \mathbb{R}. 
\end{align}
Therefore, we have $f_{\beta}'(x) \geq f_{\delta}'(x)$, $\forall x \in \mathbb{R}$.
\end{proof}

\begin{proof}[Proof of Proposition~\ref{prop:nonseparable}]
$f$ has the form \eqref{eq:GAMs}, $f$ is differentiable, and is strongly monotonic with respect to $x_{\beta}$ over $x_{\gamma}$, therefore
\begin{align}
\min_x f_{\beta}'(x) \geq \max_x f_{\gamma}'(x) \geq 0.
\end{align}
where we have assumed monotonically increasing without loss of generality in the content. 
Now at $x^*$,  $f_{\beta}'(x^*) = 0$. Therefore, $f_{\gamma}'(x)=0, \forall x$, and $f_{\gamma}(x)$ is a constant function. 
\end{proof}

\begin{proof}[Proof of Lemma~\ref{lem:binary_mono}]

We partition $\*x = (x_{\beta},x_{\gamma},\*x_{\neg})$, where $x_{\beta}$ and $x_{\gamma}$ are binary. Suppose $f$ is weakly monotonic with respect to $x_{\beta}$ over $x_{\gamma}$, then we have
\begin{align}
f(0,1,\*x_{\neg}) & \leq f(1,0,\*x_{\neg}). 
\end{align}
Note this coincides with the inequality required for strong pairwise monotonicity for the binary feature. Strong pairwise monotonicity implies weak pairwise monotonicity based on Lemma~\ref{lem:str_imply_weak}.

\end{proof}

\subsection{Remarks about Architectures of MGNAMs}\label{sec:mono_remark}

\begin{remark}
Additional consideration should be given to the case in which there is a mixture of strong and weak pairwise interactions. Suppose $f$ is strongly monotonic with respect to $x_{\delta}$ over $x_{\eta}$, we group them together as $\*x_{\beta}$. Consider the case for $\*x = (\*x_{\beta},x_{\gamma})$, where $f$ is weakly monotonic with respect to $x_{\eta}$ in $\*x_{\beta}$ over $x_{\gamma}$, then we shouldn't separate $\*x_{\beta}$ and $x_{\gamma}$ as there will be some unfair comparisons. More specifically, $\*x_{\beta}$ can take different choices of values, whereas $x_{\gamma}$ is only one-dimensional.  It follows that if there is strong pairwise monotonicity involved, then all pairwise related features should be grouped together. As a concrete example, let $x_{\delta}$ count the number of past dues with 60+ days within one year, $x_{\eta}$ count the number of past dues with 30-59 days within one year, and $x_{\gamma}$ counts the number of past dues with 30-59 days one year ago. As there is strong pairwise monotonicity between $x_{\delta}$ and $x_{\eta}$, we group them together. Suppose now we take the additive form that
\begin{align}
f(\*x) = g(x_{\delta},x_{\eta}) + h(x_{\gamma}),
\end{align}
for some differentiable functions $g$ and $h$. The weak pairwise monotonicity between $x_{\eta}$ and $x_{\gamma}$ would requires that 
\begin{align}
\frac{\partial}{\partial x_{\eta}} g(x_{\delta},y) \geq \frac{\partial}{\partial x_{\gamma}} h(y), \ \forall y,x_{\delta}.
\end{align}
Note $g$ is impacted by values of $x_{\delta}$ and $h$ is not. This is inconsistent with our intention and is an unfair comparison. 
\end{remark}

\subsection{Empirical Examples}\label{sec:empirical_example}

For all our experiments, the dataset is randomly partitioned into $75\%$ training and $25\%$ test sets. All neural networks contain 1 hidden layer with 2 units, logistic activation, and no regulation. We monitored the model selection cross-validation empirical results and observe no obvious improvement in accuracy based on in/out-of-sample results, except for the healthcare dataset due to insufficient data. We tested models with up to 20 hidden units and two hidden layers. No obvious improvement in accuracy was observed. Additionally, we checked existing literature or public codes online; our accuracy is comparable. For accuracy, we check classification errors and the area under the curve (AUC). The code is built and modified based on \cite{Vahe2023}.

\subsubsection{Finance - Credit Scoring}

A popularly used dataset is the Kaggle credit score dataset \footnote{https://www.kaggle.com/c/GiveMeSomeCredit/overview}. For simplicity, data with missing variables are removed. Past dues greater than four times are truncated. Further careful data cleanings could potentially improve model performance but are not the primary concern of this paper. Among the total 120969 observations, 8,357 (6.95$\%$) relate to the cardholders with default payments. This shows that the data are seriously imbalanced.  The dataset contains 10 features as explanatory variables:
\begin{itemize}
\item $x_1$: Total balance on credit cards and personal lines of credit except for real estate and no installment debt such as car loans divided by the sum of credit limits
\item $x_2$: Age of borrower in years
\item $x_3$: Number of times borrower has been 30-59 days past due but no worse in the last 2 years
\item $x_4$: Monthly debt payments, alimony, and living costs divided by monthly gross income 
\item $x_5$: Monthly income
\item $x_6$: Number of open loans (installments such as car loan or mortgage) and lines of credit (e.g., credit cards) 
\item $x_7$: Number of times borrower has been 90 days or more past due 
\item $x_8$: Number of mortgage and real estate loans including home equity lines of credit 
\item $x_9$: Number of times borrower has been 60-89 days past due but no worse in the last 2 years 
\item $x_{10}$: Number of dependents in the family, excluding themselves (spouse, children, etc.) 
\item $y$: Client's behavior; 1 = Person experienced 90 days past due delinquency or worse
\end{itemize}
The feature age is further excluded to avoid potential discrimination. 

\subsubsection{Criminal Justice - COMPAS}

COMPAS is a proprietary score developed to predict recidivism risk, which is used to guide bail, sentencing, and parole decisions. It has been criticized for racial bias\cite{angwin2016machine,dressel2018accuracy,tan2018distill}. A report published by ProPublica in 2016 provided recidivism data for defendants in Broward County, Florida \cite{Propublica2016COMPAS}. We focus on the simplified cleaned dataset provided in \cite{dressel2018accuracy}. Three thousand and fifty-one ($45\%$) of the 7,214 observations committed a crime within two years. This study uses a binary response variable, recidivism, as the response variable. The dataset here contains nine features, which were selected after some feature selection was conducted. 
\begin{itemize}
\item $x_1$: Races include White (Caucasian), Black (African American), Hispanic, Asian, Native American, and Others
\item $x_2$: Sex, male or female
\item $x_3$: Age
\item $x_4$: Total number of juvenile felony criminal charges
\item $x_5$: Total number of juvenile misdemeanor criminal charges
\item $x_6$: Total number of non-juvenile criminal charges
\item $x_7$: A numeric value corresponding to the specific criminal charge
\item $x_8$: An indicator of the degree of the charge: misdemeanor or felony
\item $x_9$: An numeric value between 1 and 10 corresponds to the recidivism risk score generated by COMPAS software (a small number corresponds to a low risk, and a larger number corresponds to a high risk)
\item $y$: Whether the defendant recidivated two years after the previous charge
\end{itemize}
To avoid discrimination, we further exclude races and sexes. The COMPAS score is also excluded as it is not the focus of this study and is correlated with other features, making its interpretation more difficult. As there are too few samples, we truncate the number of juveniles exceeding three. Otherwise, if monotonicity is requested, NN functions will become flat, which is not useful.

\subsubsection{Healthcare - Heart Failure Clinical Records}

This dataset \cite{ahmad2017survival,chicco2020machine} focuses on the prediction of patients' survival with heart failure in 2015. In total, there are 299 patients. The concept of fairness may be relevant here, for example, if doctors are required to decide which operation should be performed first based on the patient's condition. Death is used as the response variable in this study. This dataset contains a total of 12 features. 
\begin{itemize}
    \item $x_1$: Age
    \item $x_2$: Anaemia, a decrease of red blood cells or hemoglobin
    \item $x_3$: High blood pressure, if the patient has hypertension
    \item $x_4$: Creatinine phosphokinase
    \item $x_5$: If the patient has diabetes 
    \item $x_6$: Ejection fraction, percentage of blood leaving the heart at each contraction
    \item $x_7$: Platelets in the blood (kiloplatelets/mL)
    \item $x_8$: Sex
    \item $x_9$: Level of serum creatinine in the blood (mg/dL)
    \item $x_{10}$: Level of serum sodium in the blood (mEq/L)
    \item $x_{11}$: If the patient smokes or not
    \item $x_{12}$: Time, follow-up period (days)
    \item $y$: Death event, if the patient deceased during the follow-up period
\end{itemize}

\end{document}